\documentclass[letterpaper, 10 pt, conference]{ieeeconf}  % Comment this line out if you need a4paper

\IEEEoverridecommandlockouts                              % This command is only needed if 
\usepackage{float}
\usepackage{graphics} % for pdf, bitmapped graphics files
\usepackage{epsfig} % for postscript graphics files
\usepackage{times} % assumes new font selection scheme installed
\usepackage{makecell}
\usepackage{pifont}% http://ctan.org/pkg/pifont
\usepackage{array}
\usepackage{array}
\usepackage{tikz}
\usepackage{stmaryrd} % for special math symbols
\usetikzlibrary{positioning}
\usepackage{tabularx}
\usepackage{xcolor}
\usepackage{amsmath}
\usepackage{amsfonts}
\usepackage{booktabs}
\usepackage{mathtools}
\usepackage{stmaryrd}
\usepackage{algorithm}
\usepackage{algpseudocode}
\usepackage{tcolorbox}

\usepackage{wrapfig}
\usepackage{tikz}
\usepackage{graphicx}
\usepackage{svg}
\allowdisplaybreaks
\usepackage{amssymb}
\usepackage{moresize}
\usepackage[utf8]{inputenc}
\usepackage{etoolbox}
\usepackage{lipsum}
\usepackage{scalerel}
\usepackage{epsfig}
\usepackage{cite}
\usepackage{graphicx}
\usepackage{booktabs}
\usepackage[dvipsnames]{xcolor}
\usepackage{siunitx}
\usepackage[hidelinks]{hyperref}
\hypersetup{
  colorlinks   = true, %Colours links instead of ugly boxes
  urlcolor     = blue, %Colour for external hyperlinks
  linkcolor    = black, %Colour of internal links
  citecolor   = black %Colour of citations
}
\usepackage{tabularx}
\usepackage{amsfonts}
\usetikzlibrary{shapes.geometric, arrows}
\usetikzlibrary{calc,fit,arrows}
\newtheorem{theorem}{Theorem}[section]

\newtheorem{problem}[theorem]{Problem}

\usepackage{accents}
\newtheorem{lemma}[theorem]{Lemma}
\newtheorem{definition}[theorem]{Definition}
\usepackage{newtxmath}
\usepackage{enumerate}
\algrenewcommand\algorithmicrequire{\textbf{Input:}}
\algrenewcommand\algorithmicensure{\textbf{Output:}}
\usepackage{graphicx}
\usepackage{textcomp}
\usepackage{xcolor}
\usepackage{svg}
\def\BibTeX{{\rm B\kern-.05em{\sc i\kern-.025em b}\kern-.08em
    T\kern-.1667em\lower.7ex\hbox{E}\kern-.125emX}}

\newcommand{\bb}

\algdef{SE}[SUBALG]{Indent}{EndIndent}{}{\algorithmicend\ }%
\algtext*{Indent}
\algtext*{EndIndent}

\title{Tractable Reinforcement Learning for Full Class of Signal Temporal Logic Specifications Using Spatiotemporal Tube Reward}

\author{Vaishnavi Jagabathula$^{\ddag,1}$,P Sangeerth$^{\ddag,  1}$, Pushpak Jagtap$^1$ % <-this % stops a space
\thanks{This work was supported in part by an ANRF grant and Siemens.}
\thanks{$^\ddag$Authors contributed equally.}
\thanks{
$^1$ Centre for Cyber-Physical Systems, IISc, Bangalore, India} 
\thanks{{\tt\footnotesize\{vaishnavij,sangeerthp, pushpak\}@iisc.ac.in}}%
}

\begin{document}

\maketitle
\thispagestyle{empty}
\pagestyle{empty}

%%%%%%%%%%%%%%%%%%%%%%%%%%%%%%%%%%%%%%%%%%%%%%%%%%%%%%%%%%%%%%%%%%%%%%%%%%%%%%%%
\begin{abstract}
This paper addresses the control problem for robotic systems, including non-holonomic and underactuated platforms operating under unknown dynamics and strict actuator limits to satisfy complex high-level specifications. We denote these high-level specifications using Signal Temporal Logic (STL) and propose a novel time-aware Reinforcement Learning (RL) framework that leverages the geometric properties of Spatiotemporal Tubes (STTs). While traditional analytical STT controllers often struggle to enforce input constraints, and existing RL approaches rely on memory-intensive state history, our method natively overcomes both limitations. By mapping the logical and temporal complexities of the full class of STL into time-varying geometric boundaries, we directly constrain the multidimensional system state without relying on scalar robustness metrics. Augmenting the state space with time, we train a time-aware Soft Actor-Critic (SAC) agent using a continuous, geometry-aware reward function that eliminates the need to explicitly evaluate complex logical semantics during execution. The proposed framework offers a history-free, computationally efficient approach to learn continuous control policies that ensure robust satisfaction of specifications while strictly adhering to system input constraints.
\end{abstract}

\section{Introduction}\label{sec: introduction}

There is an increasing complexity in robotic task assignment across various industries. This requires more expressive, logic-based specifications to describe a system's desired behavior to ensure successful task completion. Signal Temporal Logic (STL) \cite{10.1007/978-3-540-30206-3_12} is a powerful and expressive framework for the quantitative monitoring of a task's spatial, and temporal specifications. Robotic missions involving ground vehicles or spacecraft that execute complex temporal tasks, such as sequential goal-reaching while avoiding environmental obstacles, can be successfully captured using STL semantics along with offering a quantitative evaluation of task performance using robustness metrics \cite{donze2010robust}. 

To synthesize controllers for these complex temporal tasks without imposing restrictive assumptions on the system's mathematical model, Reinforcement Learning (RL) \cite{sutton2018reinforcement} has gained significant attention. However, designing an RL reward function for STL specifications is notoriously challenging because STL is inherently non-Markovian; evaluating temporal satisfaction usually requires knowledge of the system's past trajectory. Existing abstraction-based synthesis tools \cite{rungger2016scots} can solve such navigation problems, but they require known system dynamics, computationally expensive grid discretizations, and yield highly conservative controllers after tedious post-processing as they natively support only Linear Temporal Logic (LTL) specifications. To apply RL directly, some methods, such as the Q-learning approach in \cite{7799279}, evaluate STL by storing state history over a finite window, which is limited by the curse of dimensionality. Other approaches attempt to maintain a Markovian state space but can only enforce a fragment of STL, such as finite-state MDPs \cite{pmlr-v120-venkataraman20a} or funnel-based reward shaping \cite{10354421}.

To bypass the complexities of evaluating logical semantics during execution, recent literature has explored mapping STL into a robustness function. A prescribed performance control (PPC) framework in \cite{STL_PPC} designs funnel controller based on robustness metrics for control-affine nonlinear systems. Expanding on this, the authors of \cite{STT_approx_free} proposed Spatiotemporal Tubes (STT) to synthesize computationally efficient controllers for the full class of STL by framing the constraints as an optimization problem. Similar analytical tube-based approaches have been developed for multi-input multi-output (MIMO) systems \cite{STT_MIMO} and specific differential drive tasks \cite{das2025full_class, das2026temporalreachavoidstaycontroldifferential}. However, these analytical STT methods rely heavily on strong structural assumptions, such as requiring the system to be control-affine or fully actuated. These controllers cannot be applied to underactuated or non-holonomic systems. These strict analytical controllers often fail due to mathematical singularities when subjected to disturbances. Furthermore, they traditionally struggle to strictly enforce physical actuator limits (input constraints). 

In this work, we address the above concerns by designing an RL-based controller that satisfies a full class of STL specifications for a general class of systems, including underactuated and non-holonomic systems, under input constraints, using the spatiotemporal tubes approach. We map the full class of STL specifications into time-varying spatiotemporal tubes and formulate a continuous, distance-based reward function. This entirely abstracts the logical and temporal complexity of the STL formula into a purely spatial representation at any given time $t$. By augmenting the system state with time, we develop a Time-Aware Soft Actor-Critic (SAC) algorithm that evaluates task satisfaction instantaneously relative to the tube bounds, eliminating the need to store trajectory history and enabling highly scalable learning. Furthermore, our RL-based framework implicitly handles non-holonomic and underactuated dynamics, handling disturbances, while strictly enforcing control input constraints via bounded actor network outputs, unlike analytical tube-based controllers. 
% The main contributions of this work are summarized as follows:
% % \begin{enumerate}
% %     \item {STT-Guided Reward Shaping}
% %     \item {History-Free time-aware RL} 
% %     \item {Robustness to complex dynamics and input constraints} 
% % \end{enumerate}
% \begin{enumerate}
%     \item \textbf{STT-Guided Reward Shaping:} We map the full class of STL specifications into time-varying spatiotemporal tubes and formulate a continuous, distance-based reward function. This entirely abstracts the logical and temporal complexity of the STL formula into a purely spatial representation at any given time $t$.
%     \item \textbf{History-Free Time-Aware RL:} By augmenting the system state with time, we develop a Time-Aware Soft Actor-Critic (SAC) algorithm that evaluates task satisfaction instantaneously relative to the tube bounds, eliminating the need to store trajectory history and enabling highly scalable learning.
%     \item \textbf{Robustness to complex dynamics and input constraints:} Unlike analytical tube-based controllers, our RL-based framework implicitly handles non-holonomic and underactuated dynamics, handling disturbances, while strictly enforcing control input constraints via bounded actor network outputs. 
% \end{enumerate}
We validate our method with three robotic case studies-(i) differential drive robot, (ii) a cartpole, and (iii) a spacecraft model.

\section{Preliminaries}
\label{preliminaries}
\textbf{Notations:}
We denote sets of nonnegative integers, positive real numbers, and non-negative real numbers, respectively, by $\mathbb{N}:=\{0,1,2,3,\dots\}$, $\mathbb{R}^+$, and $\mathbb{R}^+_0$. The symbol ${\mathbb{R}^n}$ is used to denote an $n$-dimensional Euclidean space. We define intervals in between 2 natural numbers $a,b \in \mathbb{N}$ as $[a,b]_{\mathbb{N}}$, where $a<b$. %A column vector $x \in \mathbb{R}^n$ is denoted by $x\hspace{-0.2em}=\hspace{-0.2em}[x_1;x_2;\cdots;x_n]$. The Euclidean norm of $x \in \mathbb{R}^n$ is represented by $\lVert x\rVert$. 

\subsection{System Dynamics}
In this paper, we consider a nominal mathematical model represented by a discrete-time system \cite{ogata1995discretetime}, given by
% \vspace{-1em}
\begin{equation}
\label{eqn: disc_mathematical_model}
    {\Sigma}\!: x^+ = {f}(x,u),
    % \vspace{-1em}
\end{equation}
where $x^+$ represents the state variables at the next time step, \emph{i.e.,} $x^+ \coloneq x(t + 1), \; t \in \mathbb{N}$.  Moreover, $x =[x^1; x^2; \cdots ; x^n] \in  \mathbb{R}^n$ is the state vector, $u=[u^1;u^2;\cdots; u^m] \in U\subset \mathbb{R}^m$ is the input vector in the input set $U$, and ${f}:\mathbb{R}^n \times \textit{U} \rightarrow \mathbb{R}^n$ is an unknown transition map. Furthermore, we assume that the full state vector $x$ is measurable at each time step. Building on this system model, we next define the STL specifications that serve as the satisfaction objective for our RL-based control synthesis.
% Having explained the system description, we next describe the STL specification for which we aim to synthesize a controller using RL in this work.

\subsection{Signal Temporal Logic (STL)}

Signal Temporal Logic (STL) \cite{10.1007/978-3-540-30206-3_12} is a formal language used to specify the spatial, and temporal properties of continuous-time signals. The set of STL formulae can be recursively expressed using predicates $\mathsf{p}$. Consider the predicate function $h : \mathbb{R}^n \rightarrow \mathbb{R}$, then $\mathsf{p} := \text{true}$, if $h(x) \geq 0$, and $\text{false}$ if $h(x) < 0$. An STL formula $\phi$ is recursively defined using predicates, Boolean logic, and temporal operators:
\begin{equation}
    \phi := \text{true} \mid \mathsf{p} \mid \neg\phi \mid \phi_1 \wedge \phi_2 \mid \phi_1 \mathcal{U}_{[a,b]} \phi_2, 
    \label{eqn: stl_eqn}
\end{equation}
where $\mathsf{p}$ is a predicate, and $\phi_1, \phi_2$ are STL formulae. Boolean operators for negation, and conjunction are denoted by $\neg$, and $\wedge$. $\mathcal{U}$ represent the temporal until operator in the time interval $[a, b]$ with $a, b \in \mathbb{R}_{\geq 0}$ such that $a \leq b$. The satisfaction relation $(x, t) \models \phi$ denotes if a signal $x : \mathbb{R}_{\geq 0} \rightarrow \mathbb{R}^n$, possibly a solution of (1), satisfies an STL formula $\phi$ at time $t$. The STL semantics \cite{10.1007/978-3-540-30206-3_12} for a signal $x$ is recursively given by:
\begin{align*}
    (x, t) &\models \mathsf{p} \Leftrightarrow h(x(t)) \geq 0, \\
    (x, t) &\models \neg\phi \Leftrightarrow \neg((x, t) \models \phi), \\
    (x, t) &\models \phi_1 \wedge \phi_2 \Leftrightarrow (x, t) \models \phi_1 \wedge (x, t) \models \phi_2, \\
    (x, t) &\models \phi_1 \mathcal{U}_{[a,b]} \phi_2 \Leftrightarrow \exists t_1 \in [t + a, t + b], (x, t_1) \models \phi_2 \\
    &\qquad \qquad \qquad \quad \wedge \forall t_2 \in [t + a, t_1], (x, t_2) \models {\phi_1}.
\end{align*}

The disjunction, eventually, and always operator can be derived as $\phi_1 \vee \phi_2 = \neg(\neg\phi_1 \wedge \neg\phi_2)$, $\Diamond_{[a,b]}\phi = \text{true } \mathcal{U}_{[a,b]} \phi$, and $\Box_{[a,b]}\phi = \neg\Diamond_{[a,b]}\neg\phi$. A signal $x$ satisfies an STL formula $\phi$, denoted $x \models \phi$, if and only if $(x, 0) \models \phi$. The STL robustness metric $\rho^\phi(x, t)$ \cite{donze2010robust} quantifies the degree of satisfaction: $\rho^\phi(x, t) > 0$ implies $(x, t) \models \phi$, and its magnitude reflects the strength of satisfaction or violation. The robustness semantics are defined recursively as:
\begin{subequations}
\setcounter{equation}{0}
\begin{align}
    \rho^{\mathsf{p}}(x, t) &= h(x(t)), \label{eq:3a} \\
    \rho^{\neg\phi}(x, t) &= -\rho^\phi(x, t), \label{eq:3b} \\
    \rho^{\phi_1 \wedge \phi_2}(x, t) &= \min \left(\rho^{\phi_1}(x, t), \rho^{\phi_2}(x, t)\right), \label{eq:3c} \\
    \rho^{\Diamond_{[a,b]}\phi}(x, t) &= \max_{t_1 \in [t+a, t+b]} \rho^\phi(x, t_1), \label{eq:3d} \\
    \rho^{\Box_{[a,b]}\phi}(x, t) &= \min_{t_1 \in [t+a, t+b]} \rho^\phi(x, t_1), \label{eq:3e} \\
    \rho^{\phi_1 \mathcal{U}_{[a,b]} \phi_2}(x, t) &= \max_{t_1 \in [t+a, t+b]} \min \Big( \rho^{\phi_1}(x, t_1), \nonumber \\
    &\qquad \qquad \qquad \min_{t_2 \in [t+a, t_1]} \rho^{\phi_2}(x, t_2) \Big). \label{eq:3f}
\end{align}
\end{subequations}

% We consider a finite time horizon $[0, t_f]$ over which the specification $\phi$ is to be realized, i.e., the signal $x : [0, t_f] \rightarrow \mathbb{R}^n$ is such that $x \models \phi$.

In the next section, we introduce the concept of Markov Decision Process (MDPs) before introducing the problem statement. 

\subsection{Markov Decision Process (MDP) Formulation}
To synthesize a controller for the unknown discrete-time system $\Sigma$ using reinforcement learning, we frame the sequential decision-making problem as a Markov Decision Process (MDP) \cite{sutton2018reinforcement}. An MDP is defined by the tuple $\mathcal{M}=(\mathcal{S},\mathcal{A},\mathcal{P},r,\gamma)$.

% We establish a direct equivalence between the system dynamics in \eqref{eqn: disc_mathematical_model} and this MDP framework. The continuous compact state-space $\mathcal{S}\subset\mathbb{R}^{n}$ corresponds directly to the system state-space, \textit{(i.e.)} $x \in \mathcal{S} \subset \mathbb{R}^n$. The continuous action space $\mathcal{A}:=U$ represents the allowable control inputs. The state transition probability distribution $\mathcal{P}(x^+|x,u)$ acts as a stochastic representation of the unknown system dynamics. Since the underlying system $\Sigma$ is deterministic, $\mathcal{P}(x^+|x,u)$ reduces to a Dirac delta distribution centered at the unknown transition map $f(x, u)$. Finally, $r:\mathcal{S}\times\mathcal{A}\rightarrow\mathbb{R}$ is the reward function, and $\gamma\in(0,1)$ is the discount factor. Under this formulation, the control synthesis objective translates to finding a stochastic policy $\pi(u|x)$ that selects optimal continuous control inputs $u \in \mathcal{A}$ to maximize the cumulative expected reward.

We establish a direct mapping between the system dynamics in \eqref{eqn: disc_mathematical_model} and the MDP framework. While the theoretical state vector evolves in $\mathbb{R}^n$, practical learning and robotic tasks are confined to a bounded region. Therefore, we define the MDP state space as a compact domain $\mathcal{S} \subset \mathbb{R}^n$, mapping the MDP state directly to the system state such that $x \in \mathcal{S}$. The continuous action space is defined identically to the physical input constraint set, \emph{i.e.,} $\mathcal{A} \coloneq U$. For any given state $x \in \mathcal{S}$ and action $u \in \mathcal{A}$, the transition probability distribution $\mathcal{P}(x^+|x,u)$ provides a stochastic representation of the unknown system dynamics. Since the underlying system $\Sigma$ is deterministic, $\mathcal{P}(\cdot|x,u)$ reduces to a Dirac delta distribution centered exactly at $f(x, u)$ yielding the next state $x^+ \in \mathcal{S}$. Finally, $r:\mathcal{S}\times\mathcal{A}\rightarrow\mathbb{R}$ is the customized reward function, and $\gamma\in(0,1)$ is the discount factor. Under this formulation, the control synthesis objective translates to finding a stochastic policy $\pi(u|x)$ that, for any state $x \in \mathcal{S}$, selects an optimal continuous control input $u \in \mathcal{A}$ to maximize the cumulative expected reward. Because the transition map $f(x,u)$ is unknown to the agent, the agent learns the optimal policy by interacting with the system and collecting transition samples $(x, u, x^+)$ generated by \eqref{eqn: disc_mathematical_model}. 

\subsection{Objective}
Now we introduce the problem statement considered in this work.

\begin{tcolorbox}
\begin{problem}
\label{prob:main_problem}
Given an unknown discrete-time system $\Sigma$ \eqref{eqn: disc_mathematical_model} equivalently represented by the MDP $\mathcal{M}$. The objective is to synthesize a history-free, RL-based policy $\pi(u|x)$ that ensures the satisfaction of full class of STL specifications given by \eqref{eqn: stl_eqn}.
\end{problem}
\end{tcolorbox}

In this work, we solve the above problem statement using the concept of spatiotemporal tubes, which is explained next.

\section{Proposed Methodology}
In this section, we first describe how we use concepts from spatiotemporal tube-based control to construct time-varying rewards that capture the robust satisfaction of full class of STL specifications. % Subsequently, we discuss the choice of soft-actor-critic (SAC) algorithm in the paper.

\subsection{Spatiotemporal Tubes}
To satisfy the given STL specification $\phi$, we use Spatiotemporal Tubes (STTs) which is defined next.

\begin{definition}[Spatiotemporal Tube for STL Task]
\label{defn: stt}
For an STL task $\phi$ given in \eqref{eqn: stl_eqn}, defined over the time interval $[0,t_f]_{\mathbb{N}}$, a time-varying compact set $\Gamma(t) \subset \mathbb{R}^n$ is called a valid Spatiotemporal Tube (STT) for $\phi$ if its interior is strictly non-empty for time interval $t \in [0,t_f]_{\mathbb{N}}$, and the following condition holds:
\begin{equation}
    \rho\hspace{-0.2em}^\phi\hspace{-0.2em}(\hspace{-0.1em}x\hspace{-0.1em},\hspace{-0.1em}t\hspace{-0.1em}) \hspace{-0.2em}>\hspace{-0.2em} 0,  \forall x \hspace{-0.2em}:\hspace{-0.2em} [0, t_f]_\mathbb{N} \hspace{-0.2em} \rightarrow \hspace{-0.2em}\mathbb{R}^n \text{ s.t.} x(\tau)\hspace{-0.2em} \in \hspace{-0.2em}\Gamma(\tau), \forall \tau \hspace{-0.2em}\in\hspace{-0.2em} [0, t_f]_\mathbb{N}. \label{eqn: unified_stt}
\end{equation}
\end{definition}

Spatiotemporal tubes $\Gamma(t)$ can be constructed using different geometric parameterizations- $(i)$ the hyperrectangular method \cite{STT_approx_free} or $(ii)$ the spherical ball method \cite{das2025full_class}, which share a fundamental mathematical structure as explained below:
\begin{itemize}
    \item \textbf{Hyper-rectangular STT \cite{STT_approx_free}:} In hyper-rectangular method, we construct the STT characterized by upper and lower bounds $\gamma_{u}^i(t)$ and $\gamma_{l}^i(t)$ respectively, such that $\gamma_{u}^i(t) > \gamma_{l}^i(t)$ and the spatiotemporal tube is given as $\Gamma(t) := \prod_{i=1}^n [\gamma_{l}^i(t), \gamma_{u}^i(t)]$,  for each state dimension $i$, for all $t \in [0,t_f]_{\mathbb{N}}$. %\in \{1,\ldots,n\}
    \item \textbf{Spherical STT \cite{das2025full_class} :} In the spherical method, we construct the STT characterized by a center trajectory $c: [0,t_f]_{\mathbb{N}} \rightarrow \mathbb{R}^n$ and a strictly positive radius $r: [0,t_f]_{\mathbb{N}}\rightarrow \mathbb{R}^+$ and the spatiotemporal tube is given as $\Gamma(t) := \mathcal{B}(c(t), r(t))$, for all $t \in [0,t_f]_{\mathbb{N}}$.
\end{itemize}
By defining the STT as a bounding region that inherently satisfies the robustness semantics of $\phi$, we, in the following lemma, guarantee specification satisfaction by confining the system trajectory to this region.

\begin{lemma}[STL Satisfaction via STT Confinement]
Let $\Gamma(t)$ be a valid spatiotemporal tube for the STL specification $\phi$ over the interval $[0, t_f]_\mathbb{N}$ as per Definition~\ref{defn: stt}. If a system trajectory $x(\tau)$ is constrained within the STT such that
\begin{equation}
    x(\tau) \in \Gamma(\tau), \quad \forall \tau \in [0, t_f]_\mathbb{N}, \label{eqn: confinement}
\end{equation}
then the STL specification $\phi$ is satisfied, i.e., $x \models \phi$.
\end{lemma}
\begin{proof}
    The proof follows directly from Definition 1. By the premise of the Lemma, the trajectory $x(\tau)$ satisfies the confinement condition \eqref{eqn: confinement}. Consequently, \eqref{eqn: unified_stt} of Definition~\ref{defn: stt} dictates that the robustness metric $\rho^\phi(x) > 0$. This completes the proof. 
    % By the semantics of STL, a strictly positive robustness metric implies $x \models \phi$.
\end{proof}
% Owing to space constraints, we do not discuss the design of STTs in this work. 
Since the design of STTs is a well-established technique, we refer readers to various data-driven approaches for constructing STTs, including the hyper-rectangular method \cite{STT_approx_free}, the spherical method \cite{das2025full_class}, physics-informed neural networks \cite{basu2026learningspatiotemporaltubesclass}, and others.

In the next subsection, we introduce the soft-actor-critic method as the choice for the RL algorithm to synthesize the RL-based controller.

\subsection{Soft Actor-Critic (SAC)}
While traditional value-based RL algorithms such as Deep Q-Networks (DQN) have seen success in various domains \cite{10354421}, they inherently require discrete action spaces. Discretizing the continuous input space $U$ of robotic systems often leads to the curse of dimensionality, poor scalability, and unstable learning dynamics \cite{11236998}. Therefore, to natively handle the continuous action space $\mathcal{A}$ mapped from our physical system, we adopt the Soft Actor-Critic (SAC) algorithm \cite{pmlr-v80-haarnoja18b}.

SAC is an off-policy algorithm built on the maximum-entropy reinforcement learning framework. Unlike conventional RL that maximizes only the expected cumulative reward, SAC optimizes the policy to maximize both the expected discounted return and the policy's entropy. The resulting objective is:
\begin{equation*}
\eta(\pi)=\mathbb{E}\left[\sum_{t=0}^{t_f}\gamma^{t}\left(r(x_t,u_t)+\alpha\mathcal{H}(\pi(u_t|x_t))\right)\right],
\end{equation*}
where $x_t$ and $u_t$ represent the state and input at time instant $t\in [0,t_f]_\mathbb{N}$, and the expectation is taken over trajectories generated by the policy $\pi$ and the transition dynamics $\mathcal{P}$. The term $\mathcal{H}(\pi(u_t|x_t))=-\mathbb{E}_{u_t\sim\pi}[\log\pi(u_t|x_t)]$ denotes the Shannon entropy of the policy, and $\alpha>0$ is a temperature parameter controlling the trade-off between reward maximization and exploration.

SAC employs neural networks to approximate the policy (actor network $\pi_{\phi}(u|x)$) and the soft action-value function (critic network $Q_{\theta}(x,u)$). During training, the critic network minimizes the soft action-value Bellman equation, while the actor network updates to maximize the entropy-regularized expected return, ensuring stable and robust convergence for obtaining continuous control. We refer readers to \cite{pmlr-v80-haarnoja18b} for more details about the SAC method.

In the next section, we describe the use of tube-based concepts to construct time-varying rewards that capture robust satisfaction of STL specifications.
\subsection{Construction of Reward Functions using STTs}

Given the construction of the spatiotemporal tubes for the STL specification, we design the reward function for the SAC algorithm based on the geometric parameterization of the tube. To ensure stable learning and avoid gradient masking in deep reinforcement learning, the reward is formulated to penalize the continuous distance outside the safety boundaries.

For the hyperrectangular STT method, where the tube is defined by independent lower and upper bounds $[\gamma_{l}^i(t), \gamma_{u}^i(t)]$ for each dimension $i \in \{1, \dots, n\}$, the reward function is defined as:
% \begin{equation}
%     r'(x_t,u_t, t) =  \min_{i=1}^n k_i \Big( x_{t}^i - \gamma_{l}^i(t), \gamma_{u}^i(t) - x_{t}^i \Big),
%     \label{eq:reward_rect}
% \end{equation}
\begin{equation}
    r'(x_t,u_t, t) =  \sum_{i=1}^n k_i \min \Big( x_{t}^i - \gamma_{l}^i(t), \gamma_{u}^i(t) - x_{t}^i \Big),
    \label{eq:reward_rect}
\end{equation}
where for each $i \in \{1,2,\ldots,n\}$, $k_i > 0$ is a user-defined scaling gain and $x_{t}=[x_t^1,\ldots,x_t^n]\in \mathcal{S}$ and $u_t\in \mathcal{A}$ is the state and input at time instant $t\in [0,t_f]_\mathbb{N}$.
Alternatively, for the spherical ball STT method, where the tube is characterized by a time-varying center $c(t) \in \mathbb{R}^n$ and a radius $r(t) \in \mathbb{R}^+$, the reward function is defined using the Euclidean distance:
\begin{equation}
    r'(x_t,u_t, t) = k \Big( r(t) - \|x_t - c(t)\|_2 \Big),
    \label{eq:reward_sph}
\end{equation}
where $k\in \mathbb{R}^+$, and $x_t\in \mathcal{S}$ and $u_t\in \mathcal{A}$ is the state and input at the time instant $t\in [0,t_f]_\mathbb{N}$.
Both reward formulations in \eqref{eq:reward_rect} and \eqref{eq:reward_sph} share the same fundamental property: the reward provides positive reinforcements at time $t$ if the state $x_t$ is strictly within the STT boundaries, while penalizing any state violations beyond the STT boundaries. Furthermore, the reward smoothly reaches its maximum when the state perfectly tracks the center of the tube, encouraging the RL agent to maximize robustness.

Our method fundamentally differs from funnel-based approaches \cite{10354421}, which enforce STL specifications by artificially constraining the scalar robustness value $\rho^\phi(x_t)$ within exponentially decaying bounds. Instead of operating on an abstract robustness value, our approach directly constrains the multidimensional system state $x_t \in \mathcal{S}$ within the physical geometry of the STT, $\Gamma(t)$. This provides a significant advantage for Reinforcement Learning: the temporal and logical complexities of the STL formula (such as overlapping time intervals and nested operators) are entirely abstracted into the geometric boundary of the tube. By simply evaluating its instantaneous position relative to $\Gamma(t)$, the agent ensures temporal satisfaction without needing to compute complex robustness metrics or store a history of states, effectively resolving the limitations seen in approaches like \cite{7799279,pmlr-v120-venkataraman20a}.

\subsection{Time-Aware Soft Actor-Critic}
In this section we introduce the time-aware Soft Actor-Critic algorithm which incorporates the tube-based time-dependent reward function \eqref{eq:reward_rect} or \eqref{eq:reward_sph}, which is a function of not only the state and action but also the current time $t$. We define the modified MDP as $\mathcal{M}' = (\mathcal{S}, \mathcal{A}, \mathcal{P}, r', \gamma)$, where $r': \mathcal{S} \times \mathcal{A} \times \mathbb{N} \to \mathbb{R}$ is the new time-varying reward function. The stochastic policy is now defined as $\pi': \mathcal{S} \times [0,t_f]_\mathbb{N} \to \mathcal{A}$. The Markov property of the transition probability function remains intact because the transition to $x_{t+1}$ still depends strictly on $(x_t, u_t)$, and $u_t$ depends on $x_t$ and the current time $t$. Since the reward depends on time, the soft action-value function is also explicitly dependent on time and is recursively defined as follows:
\begin{align}
    Q\hspace{-0.1em}^{\pi}\hspace{-0.1em}(\hspace{-0.1em}x_t\hspace{-0.1em},\hspace{-0.1em} u_t\hspace{-0.1em},\hspace{-0.1em} t\hspace{-0.1em}) &\hspace{-0.25em}=\hspace{-0.25em} \mathbb{E} \hspace{-0.15em}\left[\hspace{-0.15em} \sum_{k=t}^{t_f} \hspace{-0.25em}\gamma^{k\hspace{-0.1em}-\hspace{-0.1em}t}\hspace{-0.2em} \Big( \hspace{-0.1em}r'\hspace{-0.1em}(x_k\hspace{-0.1em},\hspace{-0.1em} u_k\hspace{-0.1em},\hspace{-0.1em} k) \hspace{-0.2em}+\hspace{-0.2em} \alpha \mathcal{H}(\pi'(\cdot | x_k\hspace{-0.1em},\hspace{-0.1em} k)) \Big)\hspace{-0.1em} \Bigg|\hspace{-0.1em} x_t,\hspace{-0.1em} u_t,\hspace{-0.1em} t \hspace{-0.1em}\right] \nonumber \\
    % &= \mathbb{E} \Biggl[ r'(s_t, a_t, t) + \gamma \Big( Q^{\pi}(s_{t+1}, a_{t+1}, t+1) \nonumber\\& \qquad \quad+ \alpha \mathcal{H}(\pi'(\cdot | s_{t+1}, t+1)) \Big) \Bigg| s_t, a_t, t \Biggr].
\end{align}

Furthermore, the policy optimization directly depends on time, as the actor network updates its parameters by maximizing this time-dependent soft Q-value function. To achieve this practically, the time step $t$ is concatenated with the system state $x_t$ to form an augmented state representation $\tilde{x}_t = [x_t^\top, t]^\top$, which is then fed into both the actor and critic neural networks. The proposed method is summarized in Algorithm \ref{alg:tsac}. %It is important to note that the unknown system dynamics $f(x,u)$ defined in \eqref{eqn: disc_mathematical_model} are never explicitly utilized in the policy optimization equations. Instead, the influence of $f(x,u)$ is inherently captured through the transition samples $(\tilde{x}_t, u_t, r_t, \tilde{x}_{t+1})$ collected from the environment (Line 8, Algorithm \ref{alg:tsac}), allowing the time-aware SAC agent to learn a robust control policy in a completely model-free manner.

\begin{algorithm}[htbp]
\caption{Time-Aware Soft Actor-Critic for STL Satisfaction}
\label{alg:tsac}
\begin{algorithmic}[1]
\Require Initial policy parameters $\phi$, Q-value parameters $\theta_1, \theta_2$, empty replay buffer $\mathcal{D}$, mission time horizon $t_f$
\State Set target parameters $\bar{\theta}_1 \leftarrow \theta_1, \bar{\theta}_2 \leftarrow \theta_2$
\For{each episode}
    \State Initialize environment and obtain initial state $x_0$
    \State Initialize time step $t = 0$
    \While{$t \leq t_f\textcolor{blue}{-1}$ and episode is not terminated}
        \State Construct augmented state $\tilde{x}_t = [x_t^\top, t]^\top$
        \State Sample action $u_t \sim \pi_\phi(\cdot | \tilde{x}_t)$
        \State Execute $u_t$ in the environment and observe $x_{t+1}$ generated by the unknown dynamics $f(x_t, u_t)$
        \State Calculate reward $r_t = r'(x_t, u_t, t)$ from \eqref{eq:reward_rect} or \eqref{eq:reward_sph}%using STT boundaries
        \State Construct augmented next state $\tilde{x}_{t+1} \hspace{-0.25em}=\hspace{-0.25em} [x_{t+1}^\top, t\hspace{-0.2em}+\hspace{-0.2em}1]^\top$
        \State Store transition $(\tilde{x}_t, u_t, r_t, \tilde{x}_{t+1})$ in $\mathcal{D}$
        \If{it is time to update}
            \State Sample a mini-batch of transitions $B$ from $\mathcal{D}$
            \State Update critic networks $\theta_1, \theta_2$ using soft Bellman residual on $B$
            \State Update actor network $\phi$ to maximize time-dependent $Q^{\pi}$ on $B$
            \State Update temperature parameter $\alpha$ (if using auto-tuning)
            \State Update target networks: $\bar{\theta}_i \leftarrow \tau \theta_i + (1 - \tau)\bar{\theta}_i$ for $i \in \{1, 2\}$
        \EndIf
        \State $t \leftarrow t + 1$
        % \State $x_t \leftarrow x_{t+1}$
    \EndWhile
\EndFor
\end{algorithmic}
\end{algorithm}

The next section discusses the results obtained using our time-aware SAC method.

\section{Results And Discussion}\label{results}
We validate our framework on three different case studies-(i) Turtlebot (ii) Cartpole system (iii) Spacecraft. Training and simulations are conducted on a computer equipped with an AMD Ryzen 9 5950x processor and 128 GB RAM.

\subsection{Differential drive robot}
We validate the proposed framework by conducting an experiment on the differential drive robot with the dynamics given by \cite{das2025full_class}, with linear velocity $v$ and angular velocity $\omega$ as inputs. We design an STT-based reward structure for the time-aware SAC method with $t_f=35$ sec, to enforce the following STL specification in a 2D workspace as shown in Fig~\ref{fig: turtlebot_3d}:
$\phi_1=\square_{[0,35]} {S}\land  \square_{[15,20]} G_1 \land \lozenge_{[33,35]} G_2 \land \square_{[0,20]} \neg O_1\land\square_{[20,30]} \neg O_2,$
where ${S}=[0,30]\times[0,20]$, $G_1=[16,19]\times[10,13]$, $G_2=[25,30]\times[1,4]$, $O_1=[10,12]\times[0,15], O_2=[22,23]\times[6,7]$, with input constraints $\mathcal{A}=[-15,15]\times[-15,15]$. This STL specification describes a task for the robot to always stay within the set $S$, stay in the region $G_1$ within the time interval $[15,20]$ seconds, and then finally reach $G_2$ within the time interval $[33,35]$ seconds, and avoiding obstacles the entire time duration of $[0,35]$ seconds. Degree-6 polynomial tubes for this specification were generated using \cite{das2025full_class} in approximately 1 minute. RL training phase required approximately 1 hour.

\textbf{Discussion:} The Turtlebot3 case study demonstrates the capability of our approach to handle complex sequential STL tasks ($\phi_1$) under strict non-holonomic dynamics and input constraints ($\mathcal{A}$). As summarized in Table \ref{tab:comparison}, existing control synthesis frameworks ranging from abstraction-based tools \cite{rungger2016scots} and history dependent RL \cite{7799279} to analytical tube-based methods \cite{das2025full_class, STT_approx_free} and specialized differential drive controllers \cite{das2026temporalreachavoidstaycontroldifferential} struggle to simultaneously address the full class of STL, physical actuator limits, and non-holonomic dynamics without relying on memory-intensive history storage or restrictive structural assumptions. Even though methods like \cite{das2026temporalreachavoidstaycontroldifferential} handle differential drive dynamics, they fail to meet the input constraint limit, unlike ours. By mapping the specification to time-varying geometric boundaries and utilizing bounded actor networks, our time-aware SAC method natively satisfies all these requirements in a history-free manner.

\begin{table}[htbp]
    \centering
    \caption{Comparison with other methods in literature}
    \label{tab:comparison}
    \renewcommand{\arraystretch}{1.0} 
    \begin{tabularx}{\columnwidth}{@{} X c c c c @{}}
        \hline
        \textbf{Method} & \makecell{\textbf{Full STL} \\ \textbf{Class}} & \makecell{\textbf{Input} \\ \textbf{Bounds}} & \makecell{\textbf{History-} \\ \textbf{Free}} & \makecell{\textbf{Non-} \\ \textbf{Holonomic}} \\
        \hline
        Abstraction \cite{rungger2016scots} & $\times$ & \checkmark & \checkmark & \checkmark \\
        History RL \cite{7799279} & $\times$ & \checkmark & $\times$ & \checkmark \\
        F-MDP \cite{pmlr-v120-venkataraman20a} & $\times$ & \checkmark & \checkmark & \checkmark \\
        Funnel-Analyt. \cite{LINDEMANN2021100973} & $\times$ & $\times$ & \checkmark & $\times$ \\
        Funnel-RL \cite{10354421} & $\times$ & \checkmark & \checkmark & \checkmark \\
        Analyt. STT \cite{das2025full_class, STT_approx_free} & \checkmark & $\times$ & \checkmark & $\times$ \\
        \textbf{Ours} & \textbf{\checkmark} & \textbf{\checkmark} & \textbf{\checkmark} & \textbf{\checkmark} \\
        \hline
    \end{tabularx}
\end{table}

\begin{figure}
    \centering
    \includegraphics[trim=0cm 1.5cm 0.5cm 0cm, clip, width=0.9\linewidth]{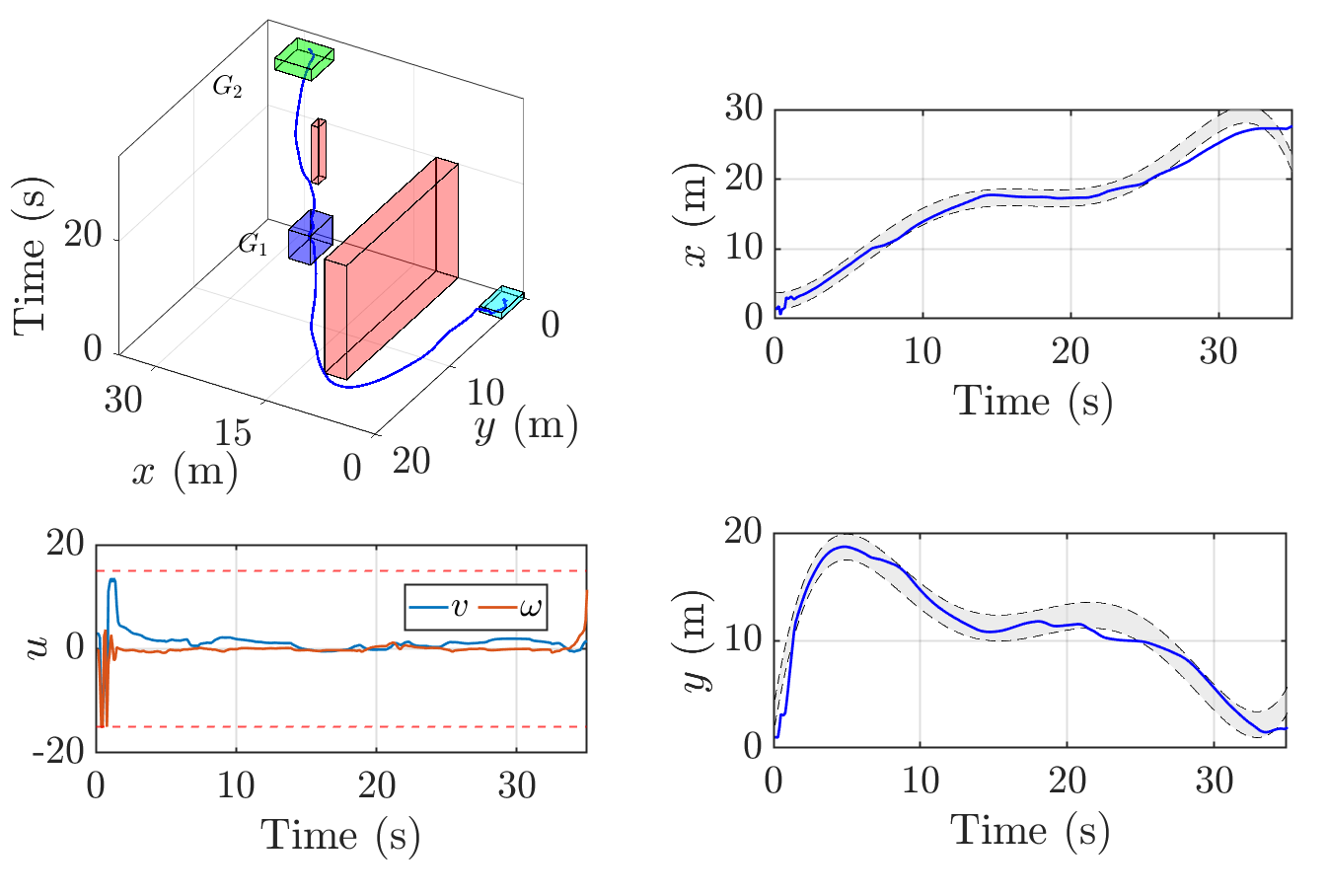}
    \caption{Turtlebot satisfying $\phi_1$, by reaching and staying in the goal region $G_1$ (blue region) within $[15,20]$ seconds and eventually region $G_2$ (Green) within $[33,35]$ seconds. The trained RL policy tries to keep the trajectories inside the circular STT, while satisfying the input constraints.}
    \label{fig: turtlebot_3d}
\end{figure}

\subsection{Cart-Pole}
This case study demonstrates the practicality of our approach by implementing the tube-based reward shaping for the underactuated cart-pole dynamics \cite{cartpole}, %given by:
% \begin{align*}
% \ddot{\theta} &= \frac{g\sin\theta-\displaystyle \cos\theta\left(\frac{F+m_p l\dot{\theta}^{\,2}\sin\theta}{m_c+m_p}\right)}{l\left(\frac{4}{3}-\frac{m_p\cos^2\theta}{m_c+m_p}\right)},\\
% \ddot{x}&=\frac{F+m_p l(\dot{\theta}^{\,2}\sin\theta- \ddot{\theta}\cos\theta)}{m_c+m_p},
% \end{align*}
 where $x$ and $\theta$ are the cart's position and the pole's angle with the vertical, respectively.%, while $\dot x$ and $\dot\theta$ denote their corresponding rates of change. The symbols $g$, $l$, $m_p$, and $m_c$ represent the acceleration due to gravity, the half-length of the pole, the mass of the pole, and the mass of the cart, respectively. 
 The control input $F$ is the horizontal force applied to the cart. We design a hyperrectangular STT-based reward structure for the time-aware SAC method with a mission horizon of $t_f=20$ s. Assuming the system starts near the downward stable equilibrium ($\theta \approx \pi$), we enforce the following STL specification:
% \begin{align}
%     \phi_2 &= \square_{[0,20]}(-6 \leq x \leq 6) \land \square_{[5,20]}\left(\vert \theta \vert \leq 15^\circ \right) \nonumber \\
%     &\quad \land \square_{[10,12]}(4 \leq x \leq 5) \land \square_{[18,20]}(-4 \leq x \leq -3). \label{eqn: spec_phi_case_study_2}
% \end{align}
   $ \phi_2 = \square_{[0,20]}(-6 \leq x \leq 6) \land \square_{[5,20]}\left(\vert \theta \vert \leq 15^\circ \right)\land \square_{[10,12]}(4 \leq x \leq 5) \land \square_{[18,20]}(-4 \leq x \leq -3).$
The objective of the STL task $\phi_2$ is to first execute a swing-up maneuver, stabilizing the pole vertically within an angle tolerance of $\pi/12$ rad ($15^\circ$) by $t=5$ s, and maintaining this balance for the remainder of the mission. Simultaneously, the cart must navigate to the target region $G_1 = [4, 5]$ during the interval $[10, 12]$ s, and subsequently move to a second region $G_2 = [-4, -3]$ during the interval $[18, 20]$ s, with an input bound of $F\in[-20,20]N$. The cart's position must strictly remain within the global workspace bound of $[-6, 6]$ throughout the entire $20$-second episode. Degree-6 polynomial tubes for this specification were generated using \cite{STT_approx_free} in approximately 25 seconds. RL training phase required approximately 15 minutes.

\textbf{Discussion:} As illustrated in Fig.~\ref{fig: cartpole}, our proposed framework successfully satisfies the complex STL specification for a highly underactuated system. A critical observation from the $x$-axis trajectory is that a slight excursion outside the STT boundaries is noticed during the $[10, 12]$ s interval. Rather than a critical violation, this highlights a significant practical advantage of our approach over traditional analytical tube-based methods. In standard barrier-based or optimization-driven frameworks, minor boundary violations due to system noise often lead to mathematical singularities (e.g., logarithmic functions yielding complex values) or strict solver infeasibility, resulting in complete controller failure. In contrast, our time-aware SAC framework maintains a continuous, geometry-aware reward gradient even when the system momentarily exits the tube. Driven by inherent reward maximization, the policy gracefully recovers, smoothly steering the system back toward the tube's center to ensure overall specification satisfaction and the reward is maximized. This inherent robustness to transient boundary violations offers a distinct operational advantage over strictly analytical approaches, such as \cite{das2026prescribedperformancecontrolunknown} where a feasibility condition needs to be satisfied. However, it is important to acknowledge that in strictly safety-critical environments, even transient excursions could theoretically lead to collisions with obstacles. Establishing formal, hard safety guarantees to completely avoid such violations in terms of spatial and temporal robustness within this RL framework is a direction for future work.

\begin{figure}[htbp]
    \centering
    \includegraphics[width=1\linewidth]{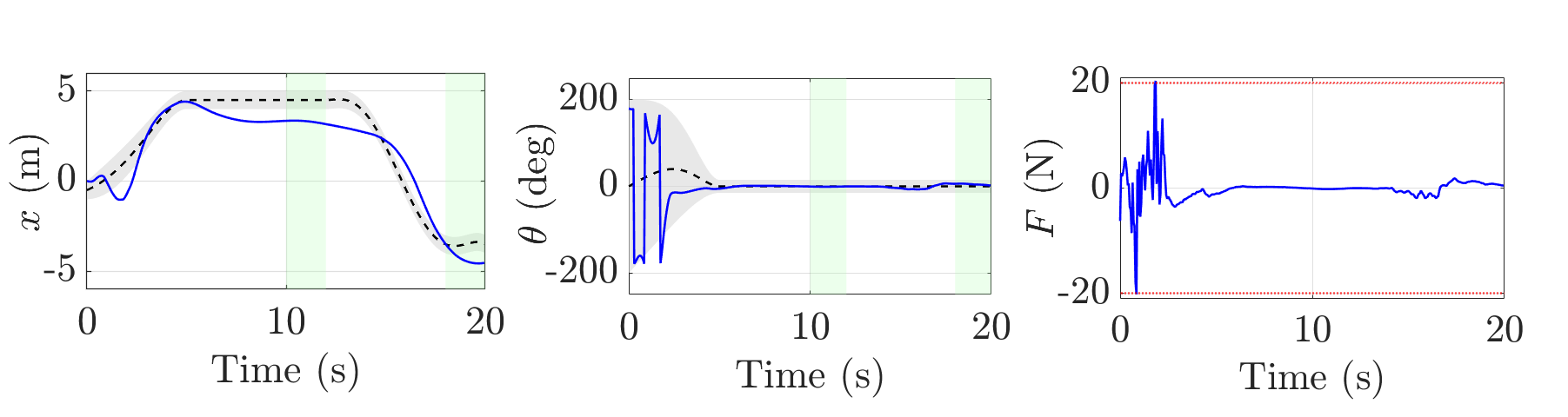}
    \caption{Cart-pole satisfying $\phi_2$ by reaching and staying in the region $G_1 = [4, 5]$ within $[10, 12]$ s, and eventually navigating to the region $G_2 = [-4, -3]$ within $[18, 20]$s, while balancing the pole angle within $[-15^\circ, 15^\circ]$ after the initial swing-up phase between [0,5] s. The trained time-aware SAC policy successfully confines the cart-pole's state within the STT to guarantee robust satisfaction of the spatial and temporal constraints.}
    \label{fig: cartpole}
\end{figure}

\subsection{Spacecraft model}
We validate the proposed framework by applying it to a spacecraft rendezvous problem, using the discrete-time dynamics from \cite{STT_approx_free}. The STL specification, adapted from the baseline study in \cite{STT_approx_free}, is given by:
$\phi_3 = S \implies \lozenge_{[7,8]}(T_1 \lor T_2) \land \lozenge_{[14,15]}G \land \square_{[0,15]}\lnot O,$
% \begin{equation}
%     \phi_3 = S \implies \lozenge_{[7,8]}(T_1 \lor T_2) \land \lozenge_{[14,15]}G \land \square_{[0,15]}\lnot O,
% \end{equation}
where $S = [0,0.6]\times[0,0.6]\times[0.4,1]$, $T_1 = [0.8,1.4]\times[1.4,2]\times[1.4,2]$, $T_2 = [2,2.6]\times[1.4,2]\times[1.4,2]$, and $G = [2.6,3.2]\times[2.6,3.2]\times[2.6,3.2]$ represent the start, intermediate target, and final goal regions, respectively. The set $O$ denotes the unsafe obstacle region. This formula dictates a sequential task: starting from $S$, the spacecraft must reach either $T_1$ or $T_2$ during the interval $[7, 8]$ s, and subsequently navigate to $G$ within $[14, 15]$ s, all while continuously avoiding $O$ over the entire $15$-second horizon. Degree-5 polynomial tubes for this specification were generated using \cite{STT_approx_free} in $0.91$ s. While the RL training phase required approximately $15$ minutes compared to the $1\mu s$ (per step substitution) needed for the analytical controller synthesis in \cite{STT_approx_free} our proposed time-aware SAC approach provides a critical operational advantage. 

\textbf{Discussion:} As shown in Fig.~\ref{fig: spacecraft_states_plot}, our trained policy produces state trajectories that successfully satisfy the STT boundaries, closely resembling the geometric response of the baseline. However, as illustrated in Fig.~\ref{fig: spacecraft_input_plots}, the analytical method in \cite{STT_approx_free} fails to respect the physical actuator limits of the spacecraft. Because our RL formulation inherently bounds the continuous action space, it successfully enforces the required control input constraint of $\lVert u \rVert \leq 700$ Nm.

\begin{figure}[htbp]
    \centering
    \includegraphics[width=1\linewidth]{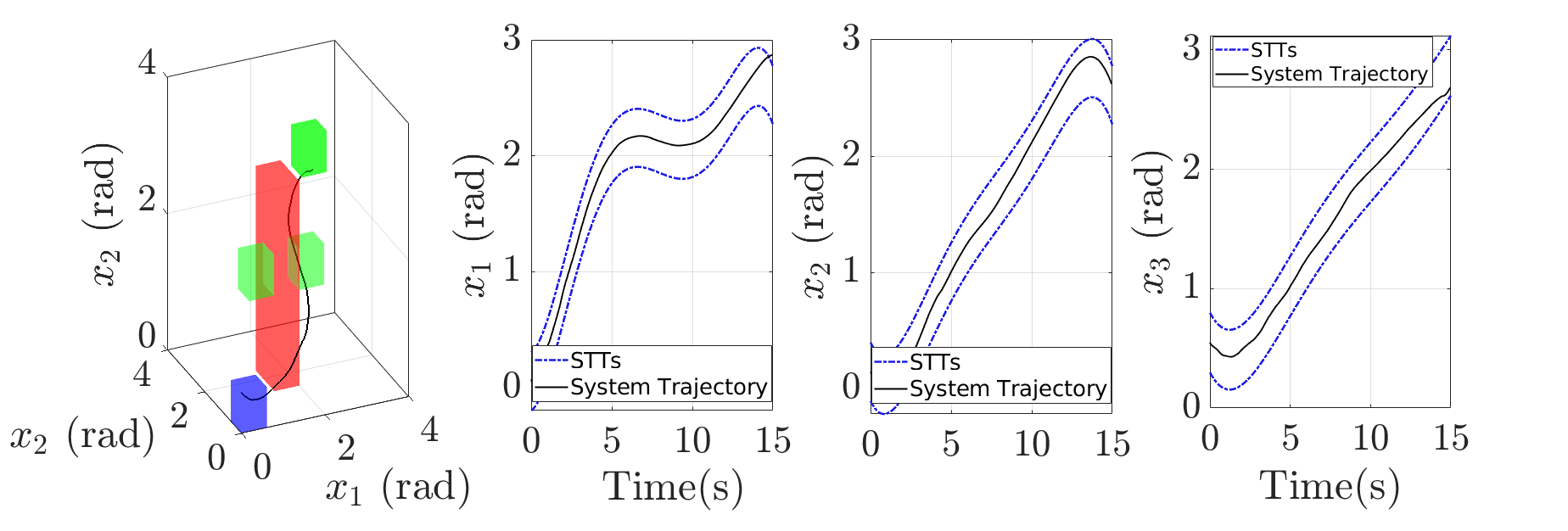}
    \caption{System trajectory and the generated STTs for the Spacecraft Case Study satisfying the specification $\phi_3$.}
    \label{fig: spacecraft_states_plot}
\end{figure}

\begin{figure}[htbp]
    \centering
    \includegraphics[width=0.5\linewidth]{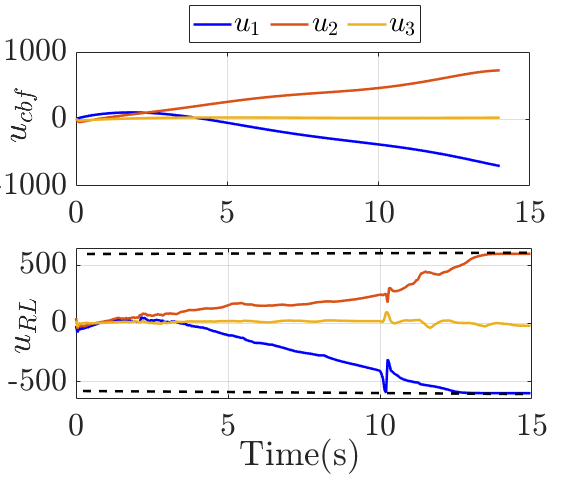}
    \caption{Comparison of the control input effort using the analytical approach from \cite{STT_approx_free} (Top) versus our proposed time-aware SAC RL approach (Bottom), demonstrating strict adherence to input constraints.}
    \label{fig: spacecraft_input_plots}
\end{figure}

\section{Conclusion and Future Work}\label{sec:conclusion}
This paper presented a history-free, time-aware RL framework using Spatiotemporal Tubes to satisfy a full class of complex STL specifications for input-constrained and underactuated systems and validated the method across diverse robotic platforms. Future work will integrate formal guarantees to preclude safety violations and extend the framework to multi-agent settings in dynamic environments.

\bibliography{reference.bib}

@article{das2026prescribedperformancecontrolunknown,
  title={Prescribed Performance Control of Unknown {E}uler-{L}agrange Systems Under Input Constraints},
  author={Das, Ratnangshu and Jagtap, Pushpak},
  journal={arXiv preprint arXiv:2507.01426},
  year={2026}
}

@inproceedings{donze2010robust,
  title={Robust satisfaction of temporal logic over real-valued signals},
  author={Donz{\'e}, Alexandre and Maler, Oded},
  booktitle={International conference on formal modeling and analysis of timed systems},
  pages={92--106},
  year={2010},
  organization={Springer}
}

@book{sutton2018reinforcement,
  title={Reinforcement learning: An introduction},
  author={Sutton, Richard S and Barto, Andrew G},
  year={2018},
  publisher={MIT press}
}

@InProceedings{pmlr-v80-haarnoja18b,
  title = 	 {Soft {A}ctor-{C}ritic: Off-Policy Maximum Entropy Deep Reinforcement Learning with a Stochastic Actor},
  author =       {Haarnoja, Tuomas and Zhou, Aurick and Abbeel, Pieter and Levine, Sergey},
  booktitle = 	 {Proceedings of the 35th International Conference on Machine Learning},
  pages = 	 {1861--1870},
  year = 	 {2018},
  volume = 	 {80},
  month = 	 {10--15 Jul},
  publisher =    {PMLR}
}

@InProceedings{10.1007/978-3-540-30206-3_12,
author="Maler, Oded
and Nickovic, Dejan",
editor="Lakhnech, Yassine
and Yovine, Sergio",
title="Monitoring Temporal Properties of Continuous Signals",
booktitle="Formal Techniques, Modelling and Analysis of Timed and Fault-Tolerant Systems",
year="2004",
publisher="Springer Berlin Heidelberg",
address="Berlin, Heidelberg",
pages="152--166",
isbn="978-3-540-30206-3"
}

@ARTICLE{10354421,
  author={Saxena, Naman and Gorantla, Sandeep and Jagtap, Pushpak},
  journal={IEEE Robotics and Automation Letters}, 
  title={Funnel-Based Reward Shaping for Signal Temporal Logic Tasks in Reinforcement Learning}, 
  year={2024},
  volume={9},
  number={2},
  pages={1373-1379},
  doi={10.1109/LRA.2023.3341775}}

@INPROCEEDINGS{7799279,
  author={Aksaray, Derya and Jones, Austin and Kong, Zhaodan and Schwager, Mac and Belta, Calin},
  booktitle={2016 IEEE 55th Conference on Decision and Control (CDC)}, 
  title={Q-Learning for robust satisfaction of signal temporal logic specifications}, 
  year={2016},
  volume={},
  number={},
  pages={6565-6570},
  doi={10.1109/CDC.2016.7799279}}

@InProceedings{pmlr-v120-venkataraman20a,
  title = 	 {Tractable {R}einforcement {L}earning of {S}ignal {T}emporal {L}ogic Objectives},
  author =       {Venkataraman, Harish and Aksaray, Derya and Seiler, Peter},
  booktitle = 	 {Proceedings of the 2nd Conference on Learning for Dynamics and Control},
  pages = 	 {308--317},
  year = 	 {2020},
  volume = 	 {120},
  month = 	 {10--11 Jun},
  publisher =    {PMLR}
}

@article{basu2026learningspatiotemporaltubesclass,
  title={Learning Spatiotemporal Tubes for Full Class of Signal Temporal Logic Tasks for Control of Unknown Systems under Input Constraints},
  author={Basu, Ahan and Das, Ratnangshu and Nath,Soumyodipta and Liu,Siyuan and Jagtap, Pushpak},
  journal={arXiv preprint arXiv:2607.07136},
  year={2026}
}

@ARTICLE{STT_approx_free,
  author={Das, Ratnangshu and Choudhury, Subhodeep and Jagtap, Pushpak},
  journal={IEEE Control Systems Letters}, 
  title={Approximation-Free Control for Signal Temporal Logic Specifications Using Spatiotemporal Tubes}, 
  year={2025},
  volume={9},
  number={},
  pages={1562-1567},
  doi={10.1109/LCSYS.2025.3579761}}

@ARTICLE{STT_MIMO,
  author={Das, Ratnangshu and Basu, Ahan and Jagtap, Pushpak},
  journal={IEEE Transactions on Automatic Control}, 
  title={Spatiotemporal Tubes for Temporal Reach-Avoid-Stay Tasks in Unknown Systems}, 
  year={2026},
  volume={71},
  number={1},
  pages={512-519},
  doi={10.1109/TAC.2025.3592723}}

@INPROCEEDINGS{STL_PPC,
  author={Lindemann, Lars and Verginis, Christos K. and Dimarogonas, Dimos V.},
  booktitle={IEEE 56th Annual Conference on Decision and Control (CDC)}, 
  title={Prescribed performance control for signal temporal logic specifications}, 
  year={2017},
  volume={},
  number={},
  pages={2997-3002},
  doi={10.1109/CDC.2017.8264095}}

@article{LINDEMANN2021100973,
title = {Funnel control for fully actuated systems under a fragment of signal temporal logic specifications},
journal = {Nonlinear Analysis: Hybrid Systems},
volume = {39},
pages = {100973},
year = {2021},
issn = {1751-570X},
doi = {https://doi.org/10.1016/j.nahs.2020.100973},
author = {Lars Lindemann and Dimos V. Dimarogonas}
}

@article{das2026temporalreachavoidstaycontroldifferential,
  title={Temporal Reach-Avoid-Stay Control for Differential Drive Systems via Spatiotemporal Tubes},
  author={Das, Ratnangshu and Basu, Ahan and Verginis, Christos and Das, Ratnangshu and Jagtap, Pushpak},
  journal={arXiv preprint arXiv:2512.05495},
  year={2026}
}

@ARTICLE{11236998,
  author={Gopalan, Aditya and Thoppe, Gugan},
  journal={IEEE Transactions on Automatic Control}, 
  title={Does {DQN} Learn?}, 
  year={2026},
  volume={71},
  number={4},
  pages={2482-2495},
  doi={10.1109/TAC.2025.3631342}}

@book{ogata1995discretetime,
  author    = {Katsuhiko Ogata},
  title     = {{D}iscrete-{T}ime {C}ontrol {S}ystems},
  edition   = {2nd},
  year      = {1995},
  publisher = {Prentice Hall},
  address   = {Englewood Cliffs, N.J.},
  isbn      = {978-0130342812}
}

@inproceedings{rungger2016scots,
  title={{SCOTS}: {A} tool for the synthesis of symbolic controllers},
  author={Rungger, Matthias and Zamani, Majid},
  booktitle={19th International Conference on Hybrid Systems: Computation and Control},
  pages={99--104},
  year={2016}
}

@article{das2025full_class,
  title={Control Barrier Functions for the Full Class of Signal Temporal Logic Tasks using Spatiotemporal Tubes},
  author={Das, Ratnangshu and Choudhury,Subhodeep and Jagtap, Pushpak},
  journal={arXiv preprint arXiv:2510.19595},
  year={2025}
}

@article{cartpole,
  title={Correct equations for the dynamics of the cart-pole system},
  author={Florian, Razvan V},
  journal={Center for Cognitive and Neural Studies (Coneural), Romania},
  volume={63},
  year={2007}
}

\end{document}